\def\year{2020}\relax
\documentclass[letterpaper]{article} % DO NOT CHANGE THIS
\usepackage{aaai20}  % DO NOT CHANGE THIS
\usepackage{times}  % DO NOT CHANGE THIS
\usepackage{helvet} % DO NOT CHANGE THIS
\usepackage{courier}  % DO NOT CHANGE THIS
\usepackage[hyphens]{url}  % DO NOT CHANGE THIS
\usepackage{graphicx} % DO NOT CHANGE THIS
\usepackage{graphicx}  % DO NOT CHANGE THIS
\usepackage{amsmath,amssymb,amsfonts}
\usepackage{mathtools}  % for floor
\usepackage{subcaption}
\usepackage{makecell, multirow}
\usepackage{color}

\newcommand{\bb}{{\bf b}}
\newcommand{\bw}{{\bf w}}

\newcommand{\BBR}{\mathbb{R}}

\newcommand{\CALL}{\mathcal{L}}

\newcommand{\TG}{\mathrm{TG}}

\newcommand{\norm}[1]{\left\Vert#1\right\Vert}
\newcommand{\ind}[1]{{\bf 1}_{#1}}

\DeclarePairedDelimiter\floor{\lfloor}{\rfloor}

\newtheorem{definition}{Definition}
\newtheorem{proposition}{Proposition}
\newenvironment{proof}[1][Proof]{\par\noindent{\bf #1.~}}
{\hspace*{\fill}\rule{5pt}{5pt}\par\vspace{0.2em}}

\title{
Plug-in, Trainable Gate for Streamlining Arbitrary Neural Networks
}
\author{
Jaedeok Kim$^1$\thanks{
Equal contribution.
}
Chiyoun Park$^1$\footnotemark[1]
Hyun-Joo Jung$^1$\footnotemark[1]
Yoonsuck Choe$^{1,2}$\\
$^1$Artificial Intelligence Center, Samsung Research, Samsung Electronics Co.\\
56 Seongchon-gil, Secho-gu,
Seoul, Korea, 06765 \\
$^{2}$Department of Computer Science and Engineering, Texas A\&M University\\
College Station, TX, 77843, USA \\
\{jd05.kim, chiyoun.park, hj34.jung\}@samsung.com, choe@tamu.edu
}

\begin{document}

\maketitle

\begin{abstract}
Architecture optimization, which is a technique for finding an efficient neural network that meets certain requirements, generally reduces to a set of multiple-choice selection problems among alternative sub-structures or parameters.
The discrete nature of the selection problem, however, makes this optimization difficult.
To tackle this problem we introduce a novel concept of a trainable gate function.
The trainable gate function, which confers a differentiable property to discrete-valued variables, allows us to directly optimize loss functions that include non-differentiable discrete values such as 0-1 selection.
The proposed trainable gate can be applied to pruning.
Pruning can be carried out simply by appending the proposed trainable gate functions to each intermediate output tensor followed by fine-tuning the overall model, using any gradient-based training methods.
So the proposed method can jointly optimize the selection of the pruned channels while fine-tuning the weights of the pruned model at the same time.
% Our proposed method pruned VGG-16 on ImageNet dataset by 90\% without fine-tuning with no accuracy drop.
Our experimental results demonstrate that the proposed method efficiently optimizes arbitrary neural networks in various tasks such as image classification, style transfer, optical flow estimation, and neural machine translation.
\end{abstract}

\section{Introduction} \label{sec:introduction}

Deep neural networks have been widely used in many applications such as image classification, image generation, and machine translation.
However, in order to increase accuracy of the models, the neural networks have to be made larger and require a huge amount of computation
\cite{he2016deep,simonyan2014very}.
Because it is often not feasible to load and execute such a large model on an on-device platform such as mobile phones or IoT devices,
various architecture optimization methods have been proposed for finding an efficient neural network that meets certain design requirements.
In particular, pruning methods can reduce both model size and computational costs effectively, but the discrete nature of the binary selection problems makes such methods difficult and inefficient \cite{he2018amc,luo2018autopruner}.

Gradient descent methods can solve a continuous optimization problem efficiently by minimizing the loss function,
but such methods are not directly applicable to discrete optimization problems because they are not differentiable.
While many alternative solutions such as simulated annealing \cite{kirkpatrick1983annealing} have been proposed to handle discrete optimization problems,
they are too cost-inefficient in deep learning because we need to train alternative choices to evaluate the sample's accuracy.

In this paper, we introduce a novel concept of a trainable gate function (TGF) that confers a differentiable property to discrete-valued variables.
It allows us to directly optimize, through gradient descent, loss functions that include discrete choices that are non-differentiable.
By applying TGFs each of which connects a continuous latent parameter to a discrete choice,
a discrete optimization problem can be relaxed to a continuous optimization problem.

Pruning a neural network is a problem that decides which channels (or weights) are to be retained.
In order to obtain an optimal pruning result for an individual model, one needs to compare the performance of the model induced by all combinations of retained channels.
While specialized structures or searching algorithms have been proposed for pruning, they have complex structures or their internal parameters need to be set manually \cite{he2018amc}.
The key problem of channel pruning is that there are discrete choices in the combination of channels, which makes the problem of channel selections non-differentiable.
Using the proposed TGF allows us to reformulate discrete choices as a simple differentiable learning problem,
so that a general gradient descent procedure can be applied, end-to-end.

Our main contributions in this paper are in three fold.
\begin{itemize}
    \item
    We introduce the concept of a TGF which makes a discrete selection problem solvable by a conventional gradient-based learning procedure.
    \item
    We propose a pruning method with which a neural network can be directly optimized in terms of the number of parameters or that of FLOPs.
    The proposed method can prune and train a neural network simultaneously, so that the further fine-tuning step is not needed.
    \item
    Our proposed method is task-agnostic so that it can be easily applied to many different tasks.
\end{itemize}

Simply appending TGFs, we have achieved competitive results in compressing neural networks with minimal degradation in accuracy.
For instance, our proposed method compresses ResNet-56 \cite{he2016deep} on CIFAR-10 dataset \cite{cifar10} by half in terms of the number of FLOPs with negligible accuracy drop.
In a style transfer task, we achieved an extremely compressed network which is more than 35 times smaller and 3 times faster than the original network.
Moreover, our pruning method has been effectively applied to other practical tasks such as optical flow estimation and neural machine translation.

By connecting discrete and continuous domains through the concept of TGF,
we are able to obtain competitive results on various applications in a simple way.
Not just a continuous relaxation, it directly connects the deterministic decision to a continuous and differentiable domain.
By doing so, the proposed method in this paper could help us solve more practical applications that have difficulties due to discrete components in the architecture.

\section{Related Work} \label{sec:related}

Architecture optimization can be considered as a combinatorial optimization problem.
The most important factors are to determine which channels should be pruned within a layer in order to minimize the loss of acquired knowledge.

These can be addressed as a problem that finds the best combination of retained channels where it requires extremely heavy computation.
As an alternative, heuristic approaches have been proposed to select channels to be pruned \cite{he2017channel,li2016pruning}.
Although these approaches provide rich intuition about neural networks and can be easily adopted to compress a neural network quickly, such methods tend to be sub-optimal for a given task in practice.

The problem of finding the best combination can be formulated as a reinforcement learning (RL) problem and then be solved by learning a policy network.
Bello et al. \cite{bello2017neural} proposed a method to solve combinatorial optimization problems including traveling salesman and knapsack problems by training a policy network.
Zoph and Le \cite{zoph2016neural} proposed an RL based method to find the most suitable architecture.
The same approach can be applied to find the best set of compression ratio for each layer that satisfies the overall compression and performance targets, as proposed in \cite{he2018amc,zhong2018prune}.
However, RL based methods still require extremely heavy computation.

To tackle the scalability issue, a differentiable approach has been considered in various research based on continuous relaxation \cite{liu2019darts,liu2017learning,louizos2017bayesian}.
To relax a discrete problem to be differentiable, Liu et al. \cite{liu2019darts} proposed a method that places a mixture of candidate operations by using softmax.
Luo and Wu \cite{luo2018autopruner} proposed a type of a self-attention module with a scaled sigmoid function as an activation function to retain channels from probabilistic decision.
However, in these methods it is essential to carefully initialize and control the parameters of the attention layers and the scaled sigmoid function.
While a differentiable approach is scalable to a large search space,
existing approaches determine the set of selected channels in a probabilistic way so that they require an additional step to decide whether to prune each channel or not.

The method we propose here allows us to find the set of channels deterministically by directly optimizing the objective function which confers a differentiable property to discrete-valued variables,
thus bypassing the additional step required in probabilistic approaches.
The proposed optimization can be performed simply by appending TGFs to a target layer
and train it using a gradient descent optimization.
The proposed method does not depend on additional parameters, so that it does not need a careful initialization or specialized annealing process for stabilization.

\section{Differentially Trainable Gate Function} \label{sec:trainable-gate}

\begin{figure}[t]
    \centerline{\includegraphics[width=1\columnwidth]{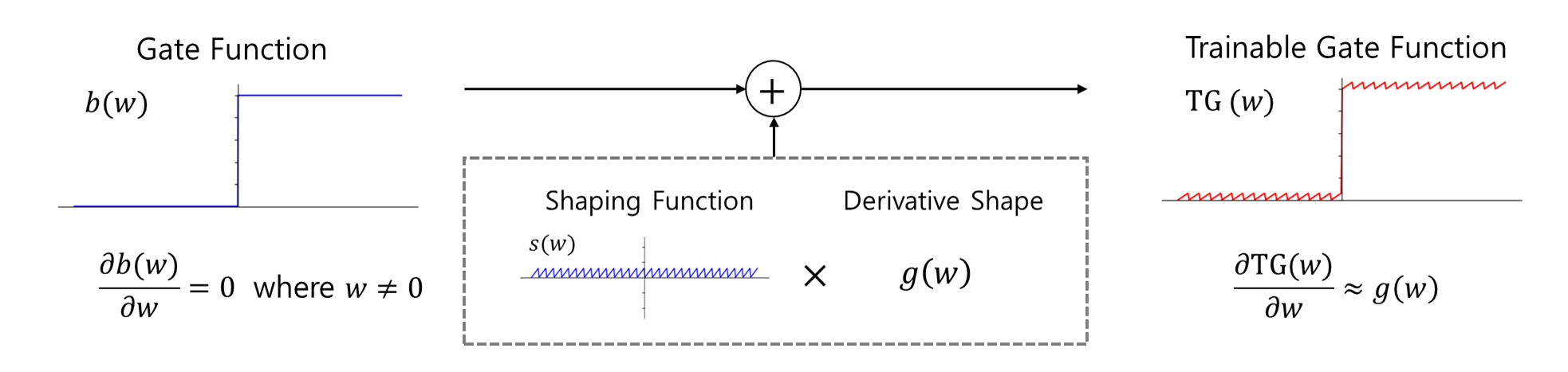}}
    \caption{
        Explanation of how to shape a gate function to be trainable.
        The original gate function has zero derivative as in the left.
        We add to the original gate function the shaping function $s(w)$ multiplied by a desirable derivative shape $g(w)$
        which changes the derivative of the gate function.
        The resulting function is a TGF that has a desirable derivative.
    }
    \label{fig:how-to-shape}
\end{figure}

Consider combinatorial optimization in a selection problem.
\begin{align} \label{eqn:comb-opt}
    \min_{\theta\in\Theta,\bb\in \{0, 1\}^n} \ \CALL({\bf \theta}, \bb)
\end{align}
where $\bb=(b_1, \cdots, b_n)$ is a vector of binary selections.
$\CALL\colon\Theta \times \{0,1\}^n\to\BBR$ is an objective function parameterized by $\theta$.
The optimization problem (\ref{eqn:comb-opt}) is a generalized form of a selection problem that can cover a parameterized loss function such as a neural network pruning problem.
In case of a pure selection problem we set the domain $\Theta$ to be a singleton set.

To make the problem differentiable, we consider $b_i$ as an output of a binary gate function $b \colon \BBR \to \{0, 1\}$ parameterized by an auxiliary variable $w_i$.
We will let $b$ be a step function for convenience.
\footnote{
Although we only consider a step function as a gate function $b$, the same argument can be easily applied to an almost everywhere differentiable binary gate function.
}
Then the optimization problem (\ref{eqn:comb-opt}) is equivalent to the following.
\begin{align} \label{eqn:comb-opt-relaxed}
    \min_{\theta\in\Theta,\bw\in\BBR^n} \ \CALL({\bf \theta}, b(w_1),\cdots,b(w_n))
\end{align}
where the problem is defined in continuous domain.
That is, if $({\bf \theta}^*, \bw^*)$ is a global minimum of (\ref{eqn:comb-opt-relaxed}),
then $({\bf \theta}^*, b(\bw^*))$ is a global minimum of (\ref{eqn:comb-opt}).

While the continuous relaxation (\ref{eqn:comb-opt-relaxed}) enables the optimization problem (\ref{eqn:comb-opt}) to be solved by gradient descent,
a gate function $b_i(\cdot)$ has derivative of zero wherever differentiable and consequently
\begin{align}
    \label{eqn:derivative}
    \frac{\partial \CALL}{\partial w_i}
    = \frac{\partial \CALL}{\partial b}
      \frac{\partial b}{\partial w_i}
    = 0,
    \textrm{ where $w_i \neq 0$}.
\end{align}
So, a gradient descent optimization does not work for such a function.

In order to resolve this issue, we consider a new type of gate function which has non-zero gradient and is differentiable almost everywhere.
Motivated by \cite{hahn2018gradient}, we first define a gradient shaping function $s\colon\BBR\to\BBR$ by
\begin{align}
	\label{eqn:shape-func}
	s^{(M)}(w) := \frac{M w - \floor{M w}}{M}
\end{align}
where $M$ is a large positive integer and $\floor{w}$ is the greatest integer less than or equal to $w$.
Note that this function has near-zero value for all $w$, and its derivative is always one wherever differentiable.
Using (\ref{eqn:shape-func}) we consider a trainable gate defined as the following (see Figure \ref{fig:how-to-shape}).

\begin{definition}
  A function $\TG^{(M)}\colon \BBR \to \BBR$ is said to be a trainable gate of a gate function $b\colon \BBR \to \{0, 1\}$ with respect to a gradient shape  $g\colon\BBR\to\BBR$ if
  \begin{align}
      \label{eqn:shaped-gate}
      \TG^{(M)}(w;g) := b(w) + s^{(M)}(w) g(w).
  \end{align}
\end{definition}
Then a trainable gate $\TG^{(M)}$ satisfies the following proposition.

\begin{proposition}
	\label{prop:shaped-gate}
	For any bounded derivative shape $g$ whose derivative is also bounded, $\TG^{(M)}(w;g)$ uniformly converges to $b(w)$ as $M\to\infty$.
    Moreover, $\TG^{(M)'}(w;g)$ uniformly converges to $g(w)$.
\end{proposition}
\begin{proof}
	By definition (\ref{eqn:shaped-gate}), it satisfies that for all $w\in\BBR$
    \begin{align*}
        |\TG^{(M)}(w;g) - b(w) |
         & = | s^{(M)}(w) g(w) | \\
         & \leq \frac{1}{M} | g(w) | \to 0,
    \end{align*}
    as $M\to\infty$.

	Also, $s'(w) = 1$ if $s(w)$ is differentiable at $w$,
    which yields for all $w\in\BBR$
    \begin{align*}
        & | \TG^{(M)}(w;g)'(w) - g(w) |  \\
        &  =   | s^{(M)'}(w) g(w) + s^{(M)}(w) g'(w) - g(w) | \\
        &  =   | s^{(M)}(w) g'(w) | \\
        & \leq \frac1M |g'(w)| \to 0,
    \end{align*}
    as $M\to\infty$.
\end{proof}

Proposition \ref{prop:shaped-gate} guarantees that the trainable gate $\TG^{(M)}(w;g)$ can approximate the original gate function $b(w)$, while its derivative still approximates the desired derivative shape $g$.
It is now possible to control the derivative of the given kernel $b(w)$ as we want and hence a gradient descent optimization is applicable to the TGF $\TG^{(M)}(w;g)$.
For convenience, we will drop the superscript and the desired function $g$ from $\TG^{(M)}(w;g)$ unless there is an ambiguity.

\subsection{Difference between Probabilistic and Deterministic Decisions}
\label{sec:syn-example}

The proposed TGF directly learns a deterministic decision during the training phase unlike existing methods.
While a probabilistic decision has been considered in existing differentiable methods \cite{jang2017categorical,liu2019darts,liu2017learning,louizos2017bayesian},
probabilistic decisions are not clear to select and hence it needs further decision steps.
Due to the on-off nature of our TGF's decision, we can include more decisive objectives, such as the number of channels, FLOPs or parameters, without requiring approximation to expectation of the distribution or smoothing to non-discrete values.

\begin{figure}[t]
    \centerline{\includegraphics[width=0.9\columnwidth]{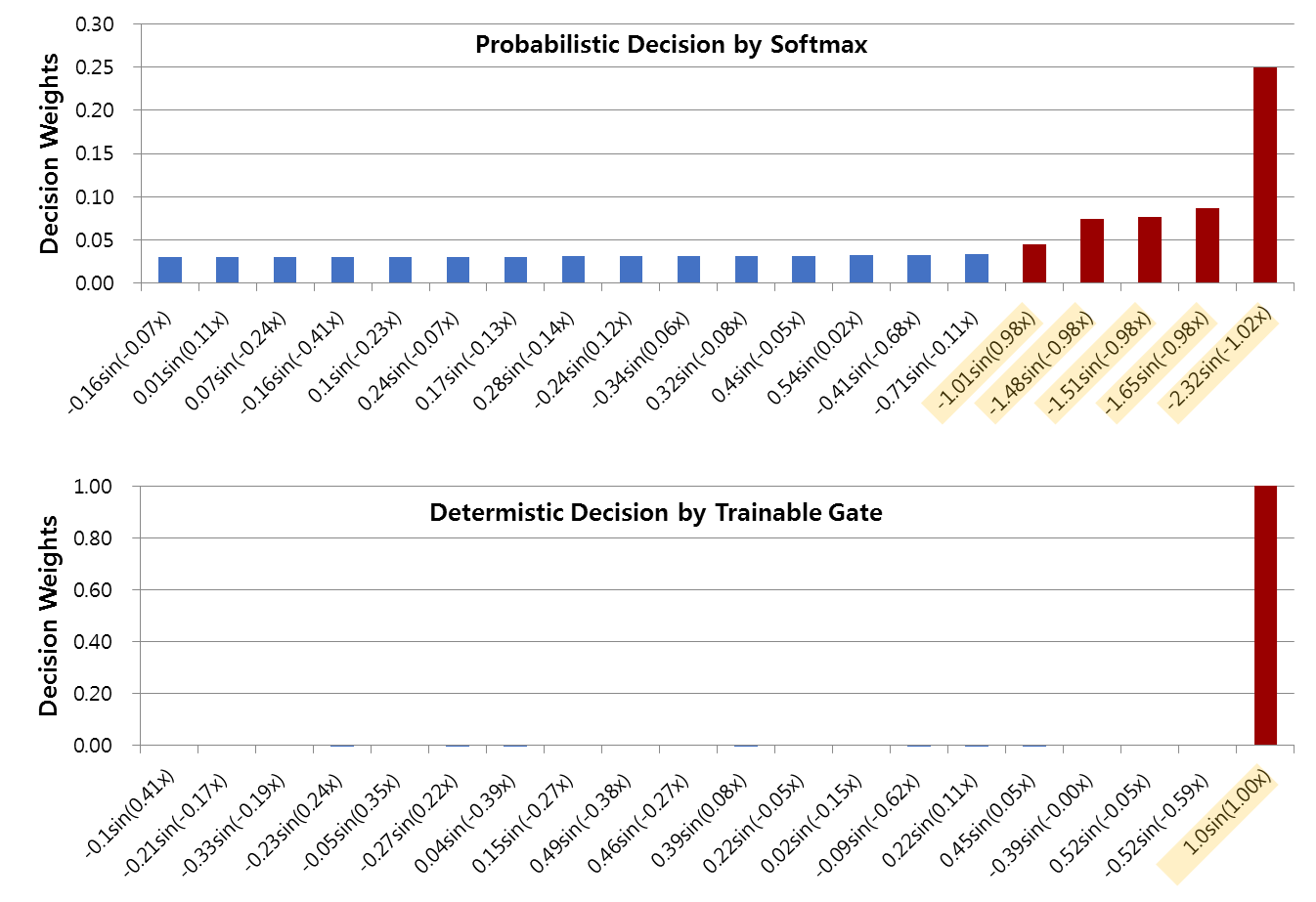}}
    \caption{
        Difference between probabilistic and deterministic decisions.
        Each bar in the above graph indicates the decision weight of a hidden node learnt by a probabilistic method.
        The $i$th label on the $x$-axis represents the learnt weights $w_i$ and $c_i$ of the $i$th hidden node.
        The graph below shows the results from a TGF where there is no redundant hidden node.
    }
    \label{fig:sine-example}
\end{figure}

We performed a synthetic experiment in order to see the difference between probabilistic and deterministic approaches.
To this end, we generate a training dataset to learn the sine function $\sin(\cdot)$.
Consider a neural network having a single fully-connected layer having 20 hidden nodes each of which adopts the sine function as its activation.
Since a training sample in our synthetic dataset is of the form $(x, \sin(x))$,
it is enough to use only one hidden node to express the relation between input $x$ and output $y$.

We consider a selection layer consisting of a continuous relaxation function with 20 hidden nodes each of which learns whether to retain the corresponding hidden node or not.
Two different types of a relaxation function are addressed: a softmax function for a probabilistic decision \cite{liu2019darts} and the proposed TGFs for a deterministic decision.

As we can see in Figure \ref{fig:sine-example},
the probabilistic method found a solution that uses more than one node.
In particular, the top 5 hidden nodes based on the decision weight have similar weight values $w_i$ and $c_i$.
In a training phase the probabilistic decision uses a linear combination of options so that error can be canceled out and as a result the selections will be redundant.
On the other hand, the deterministic decision (our proposed TGF) selects only one node exactly.
Due to the on-off nature of the deterministic decision,
the TGF learns by incorporating the knowledge of selections.
So the deterministic decision can choose effectively without redundancy.

While we have considered the binary selection problem so far,
it is straightforward to extend the concept of a trainable gate to an $n$-ary case by using $n$-simplex.
However, in order to show the practical usefulness of the proposed concept,
we will consider the pruning problem, an important application of a TGF, in which it is enough to use a binary gate function.

\section{Differentiable Pruning Method} \label{sec:diff-pruning}

In this section we develop a method to efficiently and automatically prune a neural network as an important application of the proposed TGF.
To this end, using the concept of a TGF we propose a trainable gate layer (TGL) which learns how to prune channels from its previous layer.

\subsection{Design of a Trainable Gate Layer}

\begin{figure}[t]
    \centerline{\includegraphics[width=1.0\columnwidth]{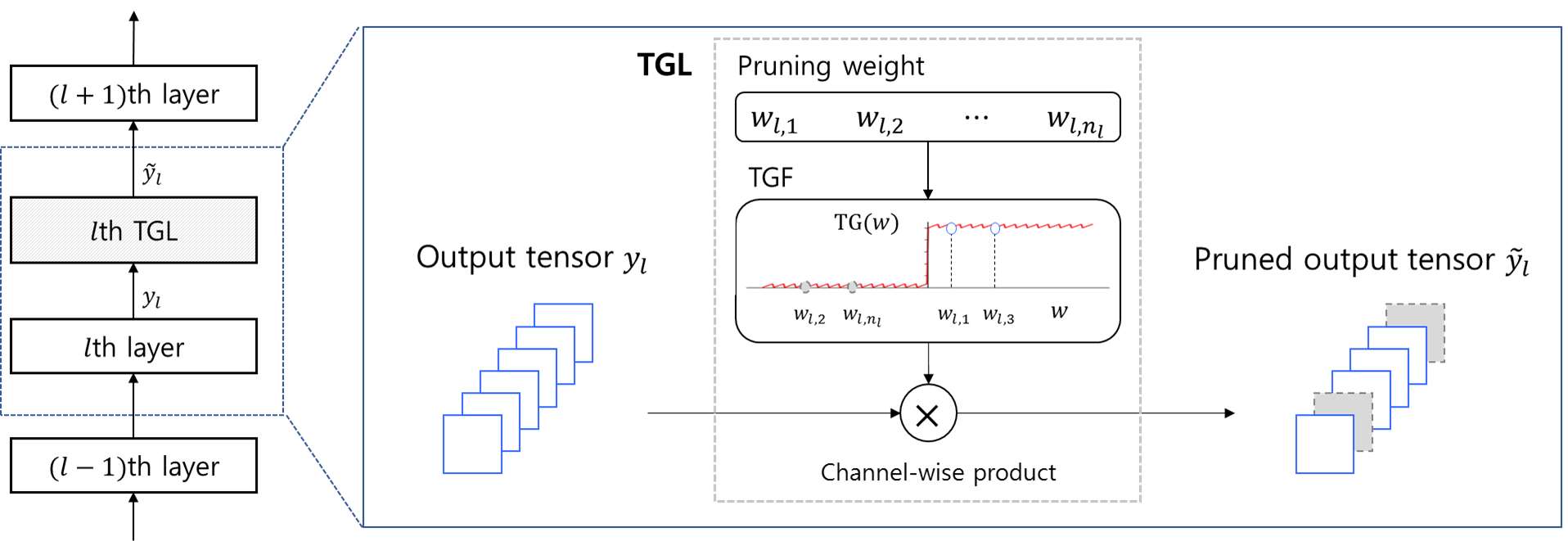}}
    \caption{
        The overview of the proposed TGL.
        The figure shows an example of a TGL appending to a convolution layer.
        The TGL consists of a set of gate functions each of which determines whether to prune or not.
        Each gate function outputs a value of 0 or 1 which is multiplied to the corresponding filter of the convolution kernel of the target layer.
    }
    \label{fig:framework}
\end{figure}

The overall framework of our proposed TGL is illustrated in Figure \ref{fig:framework}.
Channel pruning can be formulated by a function that zeros out certain channels of an output tensor in a convolutional neural network and keeps the rest of the values.
We thus design a TGL as a set of TGFs whose elements correspond to the output channels of a target layer.
Let the $l$th target layer map an input tensor $x_l$ to an output tensor $y_l:=(y_{l,1},\cdots,y_{l,n_l})$ that has $n_l$ channels using a kernel $K_l$
\begin{align*}
	y_{l,i} = K_{l,i} * x_l.
\end{align*}
for $i=1,\cdots,n_l$.
A fully connected layer uses the multiplication $\times$ instead of a convolution operation $*$, but we here simply use $*$ to represent both cases.

Let a TGL $P_l$ prune the $l$th target layer where
the TGL $P_l$ consists of trainable weights $\bw_l:=(w_{l,1},\cdots,w_{l,n_l})$ and a function $\TG(\cdot)$.
The weight $w_{l,i}$, $i=1,\cdots,n_l$, is used to learn whether the corresponding channel should be masked by zero or not.
The TGL $P_l$ masks the output tensor $y_l$ of the $l$th target layer as
\begin{align}
    \label{eqn:pruned-tensor}
	\tilde{y}_{l,i} = \TG(w_{l,i}) y_{l,i},
\end{align}
where $\tilde{y}_l:=(\tilde{y}_{l,1},\cdots,\tilde{y}_{l,n_l})$ is the pruned output tensor by $P_l$.
Since we have $y_{l,i} = K_{l,i} * x_l$, (\ref{eqn:pruned-tensor}) can be rewritten as
\begin{align*}
    \TG(w_{l,i}) y_{l,i}
    & = \TG(w_{l,i}) (K_{l,i} * x_l) \\
    & = (\TG(w_{l,i}) K_{l,i}) * x_l.
\end{align*}
So multiplying $\TG(w_{l,i})$ to $y_{l,i}$ masks the $i$th channel from the kernel $K_l$.
While $\TG(w_{l,i})$ might not be exactly zero due to the gradient shaping, its effect can be negligible by letting the value of $M$ be large enough.

From (\ref{eqn:pruned-tensor}), $\TG(w_{l,i})=0$ yields $\tilde{y}_{l,i} = 0$.
So the value of the weight $w_{l,i}$ can control the $i$th channel of the output tensor $y_l$.
If a step function $b(w)=\ind{[w>0]}$ is used as the gate function in the TGL,
$w_{l,i}<0$ implies that the TGL zeros out the $i$th channel from the $l$th layer.
Otherwise, the channel remains identical.
Hence, by updating the weights $w_l$, we can make the TGL learn the best combination of channels to prune.

The proposed channel pruning method can be extended to weight pruning or layer pruning in a straightforward way.
It can be achieved by applying the trainable gate functions for each elements of the kernel or each layer.

\subsection{Compression Ratio Control}

The purpose of channel pruning is to reduce neural network size or computational cost.
However, simply adding the above TGL without any regularization does not ensure that channels are pruned out as much as we need.
Unless there is a significant amount of redundancy in the network, having more filters is often advantageous to obtain higher accuracy, so the layer will not be pruned.
Taking this issue into account,
we add a regularization factor to the loss function that controls the compression ratio of a neural network as desired.

Let $\rho$ be the target compression ratio.
In case of reducing the number of FLOPs, the target compression ratio $\rho$ is defined by the ratio between the number of the remaining FLOPs and the total number $C_{tot}$ of FLOPs of the neural network.
The weight values $\bw:=(\bw_1,\cdots,\bw_L)$ of TGLs determine the remaining number of FLOPs, denoted by $C(\bw)$.
We want to reduce the FLOPs of the pruned model by the factor of $\rho$, that is, we want it to satisfy $C(\bw) / C_{tot} \approx \rho$.

Let $\CALL({\bf \theta},\bw)$ be the original loss function to be minimized where $\bf \theta$ denotes the weights of layers in the neural network except the TGLs.
We add a regularization term to the loss function as follows in order to control the compression ratio.
\begin{align}
    \label{eqn:loss-with-reg}
	\CALL({\bf \theta}, \bw; \rho) :=
    	\CALL({\bf \theta}, \bw)
        + \lambda \norm{\rho - C(\bw) / C_{tot}}_2^2
\end{align}
where $\norm{\cdot}_2$ denotes the $l_2$-norm and $\lambda$ is a regularization parameter.
The added regularization will ensure that the number of channels will be reduced to meet our desired compression ratio.

Note that minimization of the loss function (\ref{eqn:loss-with-reg}) does not only update the weights of TGLs but also those of the normal layers.
A training procedure, therefore, jointly optimizes the selection of the pruned channels while fine-tuning the weights of the pruned model at the same time.
In traditional approaches where channel pruning procedure is followed by a separate fine-tuning stage, the importance of each channel may change during fine-tuning stage, which leads to sub-optimal compression.
Our proposed method, however, does not fall into such a problem
since each TGL automatically takes into account the importance of each channel while adjusting the weights of the original model based on the pruned channels.
The loss function indicates that there is a trade-off between the compression ratio and accuracy.

While we have considered the number of FLOPs in this subsection,
we can extend easily to other objectives such as the number of weight parameters or channels by replacing the regularization target in (\ref{eqn:loss-with-reg}).

\section{Experimental Results} \label{sec:experiments}

In this section, we will demonstrate the effectiveness of the proposed TGF through various applications in the image and language domains.
We implemented our experiments using Keras \cite{chollet2015keras} unless mentioned otherwise.
In order to shape the derivative of the gate function $b$, a constant derivative shape $g(w)=1$ and $M=10^5$ is used and a simple random weight initialization are used.
In all experiments, only convolution or fully-connected layers are considered in calculating the number of FLOPs of a model since the other type of layers, e.g., batch normalization, requires relatively negligible amount of computation.

\subsection{Image Classification}

We used CIFAR-10 and ImageNet datasets for our image classification experiments.
We used pretrained weights of ResNet-56 on CIFAR-10 that trained from scratch with usual data augmentations (normalization and random cropping).
For VGG-16 on ImageNet we used pre-trained weights that was published in Keras.
Although we found that the accuracy of each model in our experimental setup differs slightly from the reported value,
the original pre-trained weights were used in our experiments without modification since we wanted to investigate the performance of our proposed method in terms of the accuracy drop.

\begin{figure}[t]
    \begin{subfigure}[b]{0.21\textwidth}
        \centering
        \includegraphics[width=0.97\linewidth]{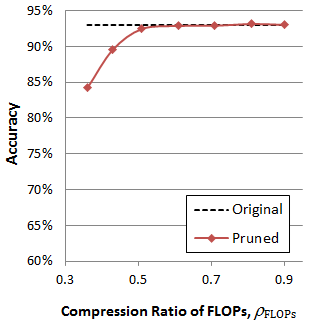}
        \caption{ResNet-56 (CIFAR-10)}
        \label{fig:noft-resnet}
    \end{subfigure}%
    \begin{subfigure}[b]{0.25\textwidth}
        \centering
        \includegraphics[width=0.91\linewidth]{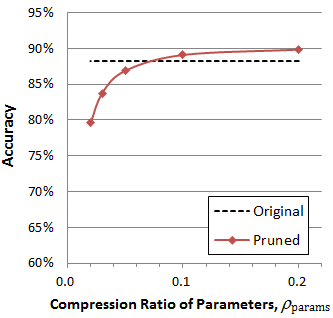}
        \caption{VGG-16 (ImageNet)}
        \label{fig:noft-vgg}
    \end{subfigure}%
    \caption{
        Pruning results without fine-tuning.
        (a) Channel pruning of ResNet-56
        % in terms of FLOPs % (#total filters=5648)
        and
        (b) Weight pruning of VGG-16
        % in terms of the number of parameters.  % (#total params=15M)
        The $x$-axis represents the compression ratio which is the ratio of the remaining of a pruned model over the original number of FLOPs or parameters.
    }
    \label{fig:noft}
\end{figure}

\textbf{Pruning without Fine-tuning}
In order to show the effect of the TGFs, we first considered a selection problem.
We kept the pre-trained weights of a model and only pruned the channels or weights without fine-tuning by appending TGLs to convolution and fully-connected layers.
That is, we fix the weights $\theta$ and optimize $\bw$ in (\ref{eqn:comb-opt-relaxed}).
Figure \ref{fig:noft-resnet} shows the results of channel pruning in ResNet-56 \cite{he2016identity} on CIFAR-10.
It can be observed that the number of FLOPs can be reduced by half without noticeable change in accuracy even when we do not apply fine-tuning, which implies that the TGFs works as expected.

We also applied weight pruning mentioned in the previous section to VGG-16 on ImageNet (Figure \ref{fig:noft-vgg}).
Even though there are more than $1.38\times10^8$ weight parameters in the model whether to be retained or not,
simply plugging-in TGLs to the model allows us to find a selection of redundant weight parameters.
Note that the accuracy of VGG-16 even increases to 89.1\% from 88.24\% when only 10\% of the parameters are retained.
This phenomenon is due to the fact that reducing the number of non-gated parameters are equivalent to applying $L_0$ regularization to the neural network, and so adding the gating function to each parameter improves the generalization property of the model further.

\begin{table}[t]
  \caption{
      Channel pruning of ResNet-56 on the CIFAR-10 and weight pruning of VGG-16 on ImageNet.
      $\textrm{acc}$ and $\Delta_{\textrm{acc}}$ denote the accuracy after pruning and the accuracy drop, respectively.
      $\rho_{\textrm{FLOPs}}$ ($\rho_{\textrm{params}}$, resp.) means the number of remaining FLOPs (parameters, resp.) over the total number of FLOPs (parameters, resp.) in each model,
      where a small value means more pruned.
      Methods:
      FP \cite{li2016pruning}, CP \cite{he2017channel},
      VCP \cite{zhao2019variational},
      AMC \cite{he2018amc},
      ADMM and PWP \cite{ye2018admm}.
  }
  \label{tbl:cifar10}
  \centering
  \begin{tabular}{c|c|c|c}
    \hline
    Model & Method & $\Delta_{\textrm{acc}}$ ($\textrm{acc}$) (\%) &  $\rho$ \\
    \hline \hline
    \multirow{6}*{
        \makecell{
            ResNet-56 \\
            ($\rho=\rho_{\textrm{FLOPs}}$)
        }
    }
      & FP & 0.02 (93.06) & 0.72 \\  % 93.04 -> 93.06
      & VCP & -0.8 (92.26) &  0.79 \\  % 93.04 -> 92.26
      & \textbf{Proposed} & \textbf{0.04} (92.7) & 0.71 \\  % 92.66 -> 92.70
      \cline{2-4}
      & CP & -1.0 (91.8) &  0.5 \\
      & AMC & -0.9 (91.9) &  0.5 \\  % 92.8 -> 90.2 (no-ft) -> 91.9 (ft) (v2 ?) AMC
      & \textbf{Proposed} & \textbf{-0.3} (92.36)  & 0.51  \\ % 93.0 -> 93.22 (v1)
    \hline
    \multirow{4}*{
        \makecell{
            VGG-16 \\
            ($\rho=\rho_{\textrm{params}}$)
        }
    }
      & ADMM & 0.0 (88.7) & 1/19 \\ % 88.7 -> 88.7  ADMM  from https://arxiv.org/pdf/1810.07378.pdf
      & \textbf{Proposed} & \textbf{0.76} (89.0) & 1/20 \\  % 88.24 -> 89.00 (0.76)
      \cline{2-4}
      & PWP & -0.5 (88.2) & 1/34 \\ % 88.7 -> 88.2 (-0.5) PWP  from https://arxiv.org/pdf/1810.07378.pdf
      & \textbf{Proposed} & \textbf{-0.06} (88.18) & 1/50 \\  % 88.24 -> 88.18 (-0.06)
    \hline
  \end{tabular}
\end{table}

\textbf{Pruning with Fine-tuning}
In the next example, we jointly optimized the weights and selection at the same time
in order to incorporate fine-tuning to the selections.
Like the previous example we appended TGLs to a model, but we jointly trained both the TGLs and the weight parameters of the model.
Table \ref{tbl:cifar10} summarizes the pruning results.
As shown in the table, our results are competitive with existing methods.
For example, in ResNet-56, the number of FLOPs is reduced by half while maintaining the original accuracy.
It is also noticeable that we achieve higher accuracy on the compressed VGG-16 model, even if the accuracy of our initial model was worse.

\begin{figure}[t]
    \centerline{\includegraphics[width=0.8\columnwidth]{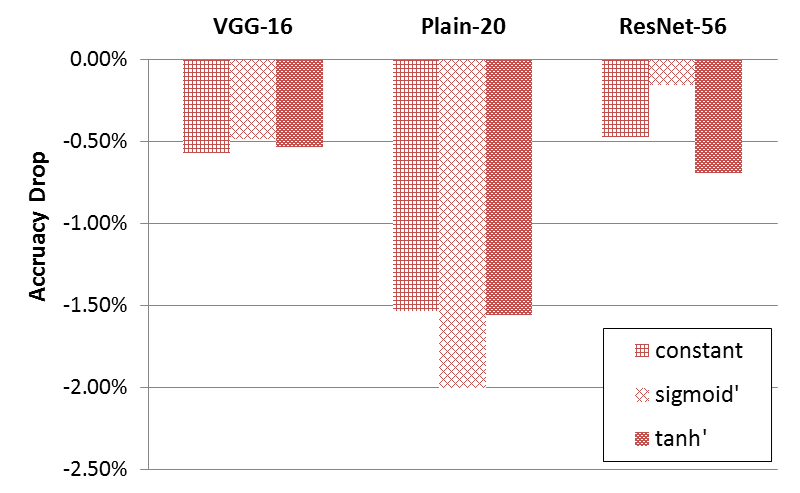}}
    \caption{
        Comparing the effect of gradient shaping in TGFs.
        {\it sigmoid'} denotes the derivative of the sigmoid function $\sigma(w)=1/(1+e^{-x})$ and {\it tanh'} denotes the derivative of the hyperbolic tangent function.
    }
    \label{fig:compare-grad-shape}
\end{figure}

While in our experiments we used a constant function to shape the derivative within the TGFs,
The proposed method can adopt any derivative shape by changing $g(w)$ in (\ref{eqn:shaped-gate}).
Figure \ref{fig:compare-grad-shape} compares the effect of different shaping functions.
It shows that the derivative shape $g$ does not affect the results critically,
which implies that our proposed method is stable over a choice of $g(w)$.
It can be concluded that a simple constant derivative shape $g(w)=1$ can be adopted without significant loss of accuracy.

\subsection{Image Generation}

We further applied our proposed pruning method to style transfer and optical flow estimation models which are the most popular applications in image generation.

\textbf{Style Transfer}
Style transfer \cite{dumoulin2016learned,gatys2015neural} generates a new image by synthesizing contents in given content image with styles in given style image.
Because a style transfer network is heavily dependent on the selection of which styles are used, it is not easy to obtain proper pre-trained weights.
So, we started from $N$-style transfer network  \cite{dumoulin2016learned} as an original architecture \footnote{
  Source code: \url{https://github.com/tensorflow/magenta/tree/master/magenta/models/image_stylization}.
} with \textit{randomly initialized weights}.
It is of course possible to start from pre-trained weights if available.

\begin{table}[t]
  \caption{
    A summary of style transfer model compression results.
    We measured the size of model files which are saved in binary format (.pb).
    The inference time was averaged over 10 frames on the CPU of a Galaxy S10.
    The frame size is 256$\times$256 pixels.
  }
  \label{tbl:style}
  \centering
  \begin{tabular}{c|c|c|c|c}
    \hline
    Model & \makecell{File size\\(MB)} & \makecell{FLOPs\\(G)} & \makecell{Params\\(M)} & \makecell{Time\\(ms)} \\
    \hline \hline
    Original & 6.9 & 10.079 & 1.674  & 549 \\ \hline
    $\rho_{ch} = 0.4$ & 1.8 & 2.470 & 0.405 & 414 \\ \hline
    $\rho_{ch} = 0.1$ & 0.2 & 0.227 & 0.020 & 160 \\
    \hline
  \end{tabular}
\end{table}

\begin{figure}[t]
    \centerline{\includegraphics[width=1\columnwidth]{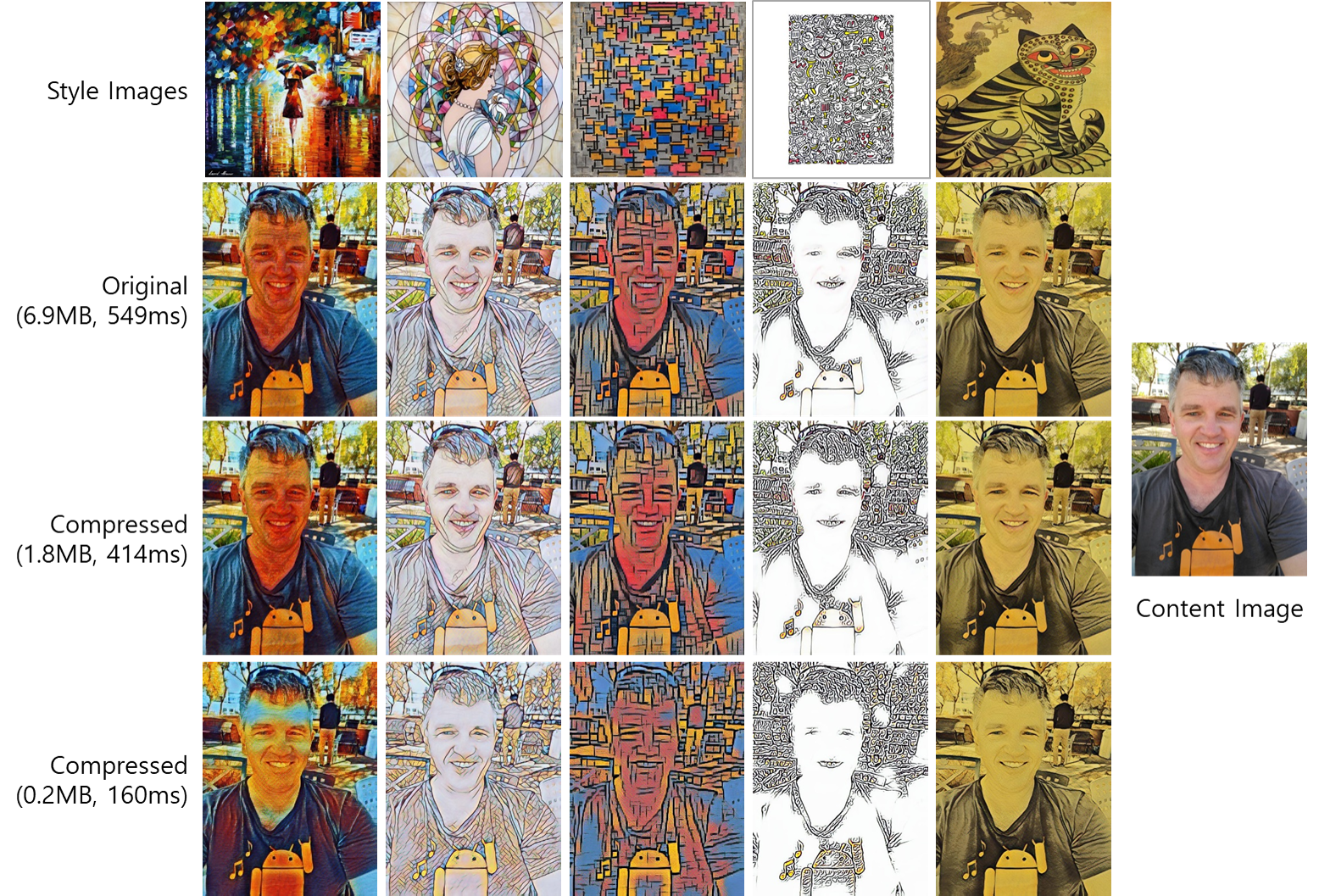}}
    \caption{
    Results of the style transfer task.
    The rightmost column represents a content image and the top row represents 5 style images.
    \textit{Second row}: Transferred images produced by the original model.
    \textit{Third row}: Transferred images produced by the compressed model ($\rho_{ch} = 0.4$).
    \textit{Fourth row}: Transferred images produced by the compressed model ($\rho_{ch} = 0.1$).}
    \label{fig:style_qual}
\end{figure}

To select which channels are retained, a TGL is plugged into the output of each convolution layer in the original architecture except the last one for preserving its generation performance.
In training phase, we used ImageNet as a set of content images and chose 5 style images (i.e., $N=5$) manually as shown in Figure \ref{fig:style_qual}.
We trained both original and compressed model from scratch for 20K iterations with a batch size of 16.
The number of pruned channels is used as a regularization factor with regularization weight $\lambda = 0.1$.

The compressed model ($\rho_{ch} = 0.1$) is $\times$34.5 times smaller than the original model in terms of file size (Table \ref{tbl:style}).
In order to see the actual inference time, we measured the inference time on the CPU of a Galaxy S10.
The compressed model ($\rho_{ch} = 0.1$) is more than $\times$3 times faster in terms of the inference time as shown although the generation quality preserved as shown in Figure \ref{fig:style_qual}.

\begin{figure}[t]
    \centerline{\includegraphics[width=\columnwidth]{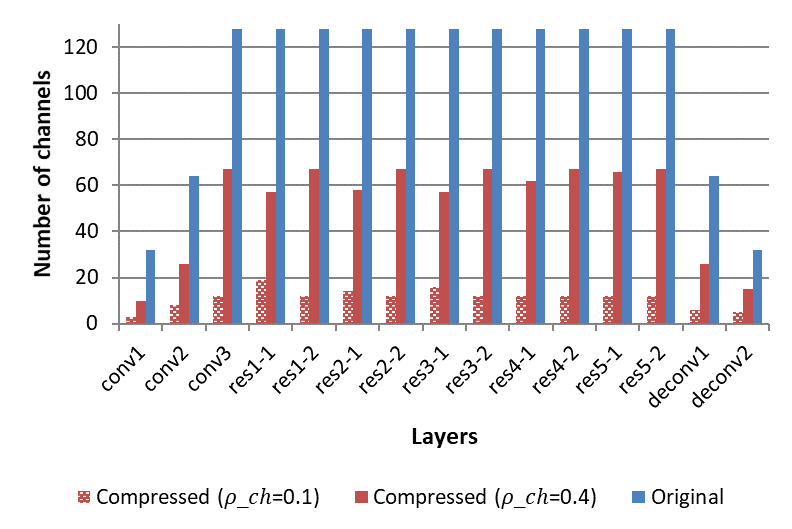}}
    \caption{
        The number of retained channels in compressed models.
        Layers that should match the dimension of outputs share a single TGL
        (conv3, res1-2, res2-2, res3-2, res4-2, and res5-2).
    }
    \label{fig:prune_num}
\end{figure}

Figure \ref{fig:prune_num} shows the number of retained channels in each layer.
The TGL does not select the number of pruned channels uniformly and it automatically selects which channels are pruned according to the objective function.
Without the further pruning steps, our proposed method can train and prune simultaneously the model with a consideration of the target compression ratio as we mentioned in the previous section.

\textbf{Optical Flow Estimation}
We next consider a task that learns the optical flow between images  \cite{dosovitskiy2015flownet,ilg2017flownet}.
In this experimentation, we used FlowNetSimple \cite{dosovitskiy2015flownet}, which is the same as FlowNetS\footnote{
    Source code: \url{https://github.com/sampepose/flownet2-tf}
} in \cite{ilg2017flownet}.
FlowNetS stacks two consecutive input images together and feed them into the network to extract motion information between these images.

\begin{table}[t]
  \caption{
      A summary of the optical flow estimation results.
      We measured the size of model files which are saved in binary format (.pb).
      For inference time, we ran both model on Intel Core i7-7700@3.60GHz CPU.
      The size of image is 384x512 pixels.
      EPE was averaged over validation data in Flying Chairs dataset.
  }
  \label{tbl:flownetS}
  \centering
  \begin{tabular}{c|c|c|c}
    \hline
    Model & File size (MB) & EPE & Time (ms) \\
    \hline \hline
    Original & 148 & 3.15 & 292.7\\ \hline
    $\rho_{ch} = 0.60$ & 54 & 2.91 & 177.8 \\ \hline
    $\rho_{ch} = 0.45$ & 30 & 3.13 & 155.9 \\
    \hline
  \end{tabular}
\end{table}

Starting from the pre-trained model, a TGL is plugged into the output of every convolution and deconvolution layers except the last one for preserving its generation performance. % model refer?
We trained the model with TGLs for 1.2M iterations with a batch size of 8.
The Adam optimizer \cite{kingma2014adam} was used with initial learning rate 0.0001 and it was halved every 200K iterations after the first 400K iterations.
As in the style transfer task, the number of pruned channels is used as a regularization factor with regularization weight $\lambda = 0.1$.
We used the Flying Chairs dataset \cite{dosovitskiy2015flownet} for training and testing.
The performance of model is measured in average end-point-error (EPE) of validation data.

Table \ref{tbl:flownetS} shows the compression results.
As we can see in the table, the compressed model ($\rho_{ch} = 0.45$) is $\times$4.93 times smaller than the original model in terms of file size and more than $\times$1.88 times faster in terms of inference time.
Note that the EPE of the compressed model (3.13) is almost same with that of the original model (3.15).
But it is only small a bit worse than EPE reported (2.71) in the paper \cite{dosovitskiy2015flownet}.

Our experimental results demonstrate that the compressed model pruned by TGL automatically finds which channels to be retained for reducing model file size, inference time, and FLOPs, while minimizing performance degradation.

\subsection{Neural Machine Translation}

\begin{table}[t]
  \caption{
      Results of pruning attention heads of a transformer model.
      The BLEU score measure on English-to-German newstest2014.
      \textit{En} denotes the total number of retained attention heads in the encoder self-attention heads.
      \textit{De} and \textit{En-De} respectively are defined for the decoder self-attention and the encoder-decoder attention.
      $\rho_{attn}$ denotes the the fraction of the total number of retained attention heads.
  }
  \label{tbl:nmt-pruning}
  \centering
  \begin{tabular}{c|c|c}
    \hline
       Model & En/De/En-De & $\Delta_{\textrm{BLEU}}$ (BLEU) \\
    \hline \hline
       Original & 48/48/48 & \quad  -- \quad  (27.32) \\
    \hline
       $\rho_{attn}=0.83$  & 30/41/48 & -0.06 \ (27.26) \\
    \hline
       $\rho_{attn}=0.39$  & 11/20/25 & -0.18 \ (27.14) \\
    \hline
  \end{tabular}
\end{table}

While we have considered various applications, all of those applications are in the image domain.
So as the last application we applied our pruning method to a neural machine translation task in the language domain.
We tried to compress the transformer model \cite{vaswani2017attention} that has been most widely used.

The transformer model consists of an encoder and a decoder.
Each layer of the encoder has multiple self-attention heads, whereas each layer of the decoder has multiple self-attention heads and multiple encoder-decoder attention heads.
To make each layer compact, we append TGFs each of which masks the corresponding attention head.
Note that unlike in the previous tasks, our pruning method can prune at a block-level, an attention head, not just at the level of a single weight or channel.

We performed WMT 2014 English-to-German translation task as our benchmark and implemented on \textit{fairseq} \cite{ott2019fairseq}.
We trained the model over 472,000 iterations from scratch.
As we can see in Table \ref{tbl:nmt-pruning},
the BLEU score of a pruned model does not degrade much.
In particular, although only 38\% of attention heads are retained, the BLEU score degrades only 0.18.
Our pruning method improved the computational efficiency in the language domain as well,
from which we can conclude that the proposed pruning method is task-agnostic.

\section{Conclusion}

In this paper, we introduced the concept of a TGF
and a differentiable pruning method as an application of the proposed TGF.
The introduction of a TGF allowed us to directly optimize the loss function based on the number of parameters or FLOPs which are non-differentiable discrete values.
Our proposed pruning method can be easily implemented by appending TGLs to the target layers,
and the TGLs do not need additional internal parameters that need careful tuning.
We showed the effectiveness of the proposed method by applying to important applications including image classification and image generation.
Despite its simplicity, our experiments show that the proposed method achieves better compression results on various deep learning models.
We have also shown that the proposed method is task-agnostic by performing various applications including image classification, image generation, and neural machine translation.
We expect that the TGF can be applied to many more applications where we need to train discrete choices and turn them into differentiable training problems.

\section*{Acknowledgements}

We would like to thank Sunghyun Choi and Haebin Shin for supports on our machine translation experiments.
Kibeom Lee and Jungmin Kwon supports the mobile phone experiments of the style transfer.

\bibliographystyle{aaai}
\bibliography{main}

\end{document}